\documentclass{article}
\usepackage{ijcai24}

\usepackage{times}
\usepackage{soul}
\usepackage{url}
\usepackage[hidelinks]{hyperref}
\usepackage[utf8]{inputenc}
\usepackage[small]{caption}
\usepackage{graphicx}
\usepackage{amsmath}
\usepackage{amsthm}
\usepackage{booktabs}
\usepackage{algorithm}
\usepackage{algorithmic}
\usepackage[switch]{lineno}

\usepackage{bm}
\usepackage{amsfonts}
\usepackage{mathtools}
\usepackage{mathrsfs}
\mathtoolsset{showonlyrefs}

\newtheorem{example}{Example}
\newtheorem{theorem}{Theorem}
\newtheorem{definition}{Definition}
\newtheorem{corollary}{Corollary}
\newtheorem{lemma}{Lemma}

\newenvironment{myproof}[1]
{\par\noindent\textbf{Proof of #1.}\ \enspace\ignorespaces\begin{allowdisplaybreaks}}
{\end{allowdisplaybreaks}\hspace{\stretch{1}}$\square$}

\title{Quality-Diversity Algorithms Can Provably Be Helpful for Optimization}

\author{
Chao Qian
\and
Ke Xue\And
Ren-Jian Wang\\
\affiliations
National Key Laboratory for Novel Software Technology, Nanjing University, Nanjing 210023, China\\
School of Artificial Intelligence, Nanjing University, Nanjing 210023, China
\emails
\{qianc, xuek, wangrj\}@lamda.nju.edu.cn
}
  
\begin{document}

\maketitle

\begin{abstract}
Quality-Diversity (QD) algorithms are a new type of Evolutionary Algorithms (EAs), aiming to find a set of high-performing, yet diverse solutions. They have found many successful applications in reinforcement learning and robotics, helping improve the robustness in complex environments. Furthermore, they often empirically find a better overall solution than traditional search algorithms which explicitly search for a single highest-performing solution. However, their theoretical analysis is far behind, leaving many fundamental questions unexplored. In this paper, we try to shed some light on the optimization ability of QD algorithms via rigorous running time analysis. By comparing the popular QD algorithm MAP-Elites with $(\mu+1)$-EA (a typical EA focusing on finding better objective values only), we prove that on two NP-hard problem classes with wide applications, i.e., monotone approximately submodular maximization with a size constraint, and set cover, MAP-Elites can achieve the (asymptotically) optimal polynomial-time approximation ratio, while $(\mu+1)$-EA requires exponential expected time on some instances. This provides theoretical justification for that QD algorithms can be helpful for optimization, and discloses that the simultaneous search for high-performing solutions with diverse behaviors can provide stepping stones to good overall solutions and help avoid local optima.
\end{abstract}

\section{Introduction}

Search algorithms, e.g., simulated annealing, Evolutionary Algorithms (EAs), and Bayesian optimization, are traditionally to find a single highest-performing solution to the problem at hand. However, for many real-world tasks, e.g., in reinforcement learning~\cite{DIAYN,neuroevolution-iclr23}, robotics~\cite{cully2015robots,robot-few-shot-qd}, and human-AI coordination~\cite{TrajeDi,fcp}, there is need for finding a set of high-performing solutions that exhibit varied behaviors. A behavior space can be characterized by dimensions of variation of interest to a user. For example, if searching for robot gaits that can run fast, the dimensions of interest can be the portion of time that each leg is in contact with the ground. Illuminating high-performing solutions of each area of the entire behavior space reveals relationships between dimensions of interest and performance, which can help improve the robustness in complex environments, e.g., help robots adapt to damage (i.e., discover a behavior that compensates for the damage) rapidly~\cite{cully2015robots}.

Quality-Diversity (QD) algorithms~\cite{qd-frontier,qd-survey-framework,qd-survey-optimization} are a subset of EAs, and have emerged as a potent optimization paradigm for the challenging task of finding a set of high-quality and diverse solutions. They follow a typical evolutionary process, which iteratively select some parent solutions from the population (also called archive in the QD area), apply variation operators (e.g., crossover and mutation) to produce offspring solutions, and use these offspring solutions to update the population. Lehman and Stanley~\shortcite{lehman2011evolving} proposed the QD algorithm, Novelty Search
and Local Competition (NSLC), by maximizing a local competition objective (e.g., measured by the number of nearest neighbors of a solution worse than itself) and a novelty objective (e.g., measured by the average distance of a solution to its nearest neighbors) simultaneously. The Multi-Objective Landscape Exploration
(MOLE) algorithm~\cite{clune2013evolutionary} is similar to NSLC, except that a global performance objective, instead of the local competition objective, is considered. 

Mouret and Clune~\shortcite{mouret2015illuminating} proposed the Multi-dimensional Archive of Phenotypic Elites (MAP-Elites) algorithm to achieve the goal of finding a set of high-quality and diverse solutions more straightforwardly. MAP-Elites maintains an archive by discretizing the behavior space into cells and storing at most one solution in each cell. The generated solutions compete only when their behaviors belong to the same cell, and MAP-Elites tries to fill in the cells with as high-quality solutions as possible. MAP-Elites performs better than NSLC and MOLE empirically~\cite{mouret2015illuminating}, and has become the most popular QD algorithm. Many advanced variants have also been proposed to improve the sample efficiency of MAP-Elites, e.g., refining parent selection strategies~\cite{qd-survey-framework,mc-me,edocs,nss} and variation operators~\cite{PGA-ME,dqd,DQD-RL,qd-pg,dcg-me}, improving archive structure~\cite{cvt-me,AURORA-CSC}, using surrogate models~\cite{model-based-qd,dsage}, and applying cooperative coevolution~\cite{ccqd}. These algorithms have achieved encouraging performance in various reinforcement learning tasks, such as exploration~\cite{NSR-ES,first-return}, robust training~\cite{one-solution,romance}, and environment generation~\cite{qd4ge,warehouse}. Recently, Wang et al.~\shortcite{wang2024resource} have tried to improve the resource efficiency of QD algorithms for the first time.

Besides generating a diverse set of high-performing solutions, extensive experiments (e.g., on evolving locomoting virtual creatures~\cite{lehman2011evolving}, producing modular neural networks~\cite{mouret2015illuminating}, 
designing simulated and real soft robots~\cite{mouret2015illuminating}, encouraging exploration of agents in deceptive~\cite{NSR-ES,dvd} and sparse reward~\cite{ME-ES} environments, and solving traveling thief problems~\cite{qd-ttp}) have shown that QD algorithms can find a better overall solution than traditional search algorithms which focus on finding the single highest-performing solution. That is, if only looking at the best achieved objective value, QD algorithms are still better. Even when searching for behavioral novelty only while ignoring the objective, the empirical results on maze navigation and biped walking show that such a novelty search significantly outperforms objective-based search~\cite{lehman2011abandoning}. The better optimization ability of QD algorithms may be due to the better exploration of the search space, which helps avoid local optima and thus find better objective values.

Compared to their wide application, the theoretical analysis of QD algorithms is far behind. To the best of our knowledge, there have existed only two pieces of works on running time analysis of QD algorithms. Nikfarjam et al.~\shortcite{nikfarjam2022analysis} analyzed the MAP-Elites algorithm for solving the knapsack problem. They proved that by using carefully designed two-dimensional behavior spaces, MAP-Elites can simulate dynamic programming behaviors to find an optimal solution within expected pseudo-polynomial time. Bossek and Sudholt~\shortcite{bossek2023runtime} further proved that for maximizing monotone submodular functions with a size constraint, MAP-Elites using the number of items contained by a set as the behavior descriptor can achieve a $(1-1/e)$-approximation ratio in expected polynomial time; for the minimum spanning tree problem, by using the number of connected components as the behavior descriptor, MAP-Elites can find a minimum spanning tree in expected polynomial time. Bossek and Sudholt~\shortcite{bossek2023runtime} also analyzed MAP-Elites for solving any pseudo-Boolean problem $f: \{0,1\}^n \rightarrow \mathbb{R}$. The number of 1-bits of a solution is used as the behavior descriptor, and the $i$-th cell of the behavior space stores the best found solution with the number of 1-bits belonging to $[(i-1)k,ik-1]$, where $k$ (dividing $n+1$) is a granularity parameter. They proved the upper bound on the expected running time of MAP-Elites until covering all cells. For the specific pseudo-Boolean problem OneMax which aims to maximize the number of 1-bits of a solution, they proved that the derived bound is tight, even for the time to find the optimal solution within each cell. There are also some preliminary studies trying to understand the behavior of novelty search from a theoretical perspective. Doncieux et al.~\shortcite{doncieux2019novelty} hypothesized that novelty search asymptotically behaves like uniform sampling in the behavior space, and Wiegand~\shortcite{wiegand2023preliminary} showed that novelty search will converge in the archive space by utilizing $k$ nearest neighbor theory.

In this paper, we try to theoretically examine the fundamental question: whether QD algorithms can be helpful for optimization, i.e., finding better objective values? For this purpose, we compare the popular QD algorithm MAP-Elites with $(\mu+1)$-EA via rigorous running time analysis. $(\mu+1)$-EA maintains a population of $\mu$ solutions, and focuses on finding better objective values by iteratively generating an offspring solution via mutation and using it to replace the worst solution in the population if it is better. Note that for fair comparison, the population size $\mu$ of $(\mu+1)$-EA is set to the number of cells of the archive of MAP-Elites, and all their parameters are set to the same. The only difference is that for MAP-Elites, the generated offspring solution only competes with the solution in the same cell of the archive, while for $(\mu+1)$-EA, the offspring competes with all solutions in the population, and replaces the worst one if it is better. Thus, their performance gap can be purely attributed to the QD algorithms' framework of searching for a set of high-quality and diverse solutions. 

We consider two NP-hard problem classes with wide applications. For monotone approximately submodular maximization with a size constraint~\cite{das2011submodular}, we prove that MAP-Elites achieves the optimal polynomial-time approximation ratio of $1-e^{-\gamma}$~\cite{harshaw2019submodular}, where $\gamma \in [0,1]$ measures the closeness of a set function $f$ to submodularity (\textbf{Theorem~\ref{theo-nonsubmodular}}). When $f$ is submodular, $\gamma=1$, and the approximation ratio of MAP-Elites is $1-1/e$ (\textbf{Corollary~\ref{corollary-submodular}}); while there exists an instance where $(\mu+1)$-EA requires at least expected exponential time to achieve an approximation ratio of nearly $1/2$ (\textbf{Theorem~\ref{theo_ea_mc}}). For the set cover problem with $m$ elements~\cite{feige1998threshold}, we prove that MAP-Elites achieves the asymptotically optimal approximation ratio of $\ln m +1$ in expected polynomial time (\textbf{Theorem~\ref{theo-sc}}), while there exists an instance where $(\mu+1)$-EA requires at least expected exponential time to achieve an approximation ratio smaller than $2^{m+1}/m$ (\textbf{Theorem~\ref{theo_ea_sc}}). These results clearly show the advantage of MAP-Elites over $(\mu+1)$-EA, providing theoretical justification for that QD algorithms can be helpful for optimization. The analyses disclose that the simultaneous search of QD algorithms for high-performing solutions at diverse regions in the behavior space plays a key role, which provides stepping stones to high-performing solutions in a region that may be difficult to be discovered if searching only in that region (which may get trapped in local optima).

The rest of the paper first introduces the studied algorithms, i.e., MAP-Elites and $(\mu+1)$-EA, and then gives the theoretical analysis on the two problem classes, monotone approximately submodular maximization with a size constraint and set cover, respectively. Finally, we conclude this paper.

\section{Preliminaries}

In this section, we will first introduce the MAP-Elites algorithm and $(\mu+1)$-EA, respectively, and then introduce their settings used in our analysis for fairness and how to measure their performance. 

\subsection{MAP-Elites Algorithm}

MAP-Elites is the most popular QD algorithm~\cite{mouret2015illuminating,cully2015robots}. Given an objective function $f: \mathcal{X} \rightarrow \mathbb{R}$ to be maximized, where $\mathbb{R}$ denotes the set of reals, MAP-Elites tries to find a set of high-quality (i.e., having large $f$ values) and diverse solutions, instead of a single highest-performing solution (i.e., a solution having the largest $f$ value). Specifically, a user chooses $m$ dimensions of variation of interest that define a behavior space, and MAP-Elites aims to create a map of high-performing solutions at each
point in the behavior space. The behavior space can be characterized by a behavior descriptor function $\bm{b}: \mathcal{X} \rightarrow \mathbb{R}^m$. For a solution $\bm{x} \in \mathcal{X}$, $\bm{b}(\bm{x})$ returns an $m$-dimensional vector describing $\bm{x}$'s behaviors.

In our theoretical analysis, we use the simple default version of MAP-Elites~\cite{mouret2015illuminating}, as presented in Algorithm~\ref{algo:map-elites}. For each behavior descriptor vector $\bm{b}(\cdot)$, an archive $M$ uses a cell (denoted as $M(\bm{b}(\cdot))$) to store the highest-performing solution found so far. Note that multiple solutions may have the same behavior descriptor vector. If no solution is found for a cell in the archive $M$, the cell is empty. The archive $M$ is initialized to be empty in line~1. The first $I$ solutions are randomly generated (lines~4--5). After that, a solution $\bm{x}$ is randomly chosen from the archive $M$ (line~7), and mutated to generate an offspring solution $\bm{x}'$ (line~8) in each iteration. Each generated solution $\bm{x}'$ will be used to update the corresponding cell (i.e., $M(\bm{b}(\bm{x}'))$) in $M$, as shown in lines~10--12. If the cell is empty (i.e., $M(\bm{b}(\bm{x}'))=\emptyset$), $\bm{x}'$ will be directly placed in the cell. If $\bm{x}'$ is higher-performing than the current occupant of the cell (i.e., $f(\bm{x}')>f(M(\bm{b}(\bm{x}')))$), the occupant will be replaced by $\bm{x}'$. Note that according to user preference or available computational resources, multiple behavior descriptor vectors may be merged into one cell in practice.

\begin{algorithm}[t]\caption{MAP-Elites Algorithm}\label{algo:map-elites}
\textbf{Input}: an objective function $f: \mathcal{X} \rightarrow \mathbb{R}$, and a behavior descriptor function $\bm{b}: \mathcal{X} \rightarrow \mathbb{R}^m$\\
\textbf{Parameter}: \#initial solutions $I$ 
    \begin{algorithmic}[1]
   \STATE Initialize an empty archive $M \gets \emptyset$;
   \STATE $iter=0$;
    \STATE \textbf{repeat}
    \STATE \quad \textbf{if} $iter < I$ \textbf{then}
    \STATE \qquad Choose $\bm x'$ from $\mathcal{X}$ uniformly at random
    \STATE \quad \textbf{else}
    \STATE \qquad Choose $\bm x$ from $M$ uniformly at random;
    \STATE \qquad Create $\bm{x}'$ by mutating $\bm{x}$
    \STATE \quad \textbf{end if}
    \STATE \quad \textbf{if} {$M(\bm{b}(\bm{x}'))=\emptyset$} or {$f(\bm{x}')>f(M(\bm{b}(\bm{x}')))$} \textbf{then}
    \STATE \qquad $M(\bm{b}(\bm{x}'))=\bm{x}'$
    \STATE \quad \textbf{end if}
    \STATE \quad $iter=iter+1$
    \STATE \textbf{until} some criterion is met
    \STATE \textbf{return} $M$
    \end{algorithmic}
\end{algorithm}

\subsection{$(\mu+1)$-EA}

$(\mu+1)$-EA is a typical EA focusing on finding solutions with larger $f$ values, which has been widely used in theoretical analyses of EAs, e.g.,~\cite{jansen2001utility,witt2006runtime,storch2008choice,qian2021analysis}. As presented in Algorithm~\ref{algo:ea}, $(\mu+1)$-EA maintains a set of $\mu$ solutions, i.e., employs a population size of $\mu$. It first initializes a population $P$ by randomly generating $\mu$ solutions in line~1. Then, in each iteration, it generates one new solution $\bm{x}'$ by mutating a solution $\bm{x}$ randomly chosen from the population $P$ (lines~3--4), and uses $\bm{x}'$ to replace the worst solution in $P$ if $\bm{x}'$ is better (lines~5--8). 

\begin{algorithm}[t]\caption{$(\mu+1)$-EA}\label{algo:ea}
\textbf{Input}: an objective function $f: \mathcal{X} \rightarrow \mathbb{R}$\\
\textbf{Parameter}: population size $\mu$ 
    \begin{algorithmic}[1]
    \STATE Initialize a population $P$ by choosing $\mu$ solutions from $\mathcal{X}$ uniformly at random;
    \STATE \textbf{repeat}
    \STATE \quad Choose $\bm x$ from $P$ uniformly at random;
    \STATE \quad Create $\bm{x}'$ by mutating $\bm{x}$;
    \STATE \quad Let $\bm{y}=\arg\min_{\bm{y}\in P}f(\bm{y})$; ties are broken randomly;
    \STATE \quad \textbf{if} {$f(\bm{x}')>f(\bm{y})$} \textbf{then}
    \STATE \qquad $P \gets (P \setminus \{\bm {y} \}) \cup \{\bm {x}'\}$
    \STATE \quad \textbf{end if}
    \STATE \textbf{until} some criterion is met
    \STATE \textbf{return} $P$
    \end{algorithmic}
\end{algorithm}

\subsection{Algorithm Settings and Performance Measurement}\label{subsec-algo-setting}

To theoretically show that QD algorithms can be helpful for optimization (i.e., finding a single, best solution), we will compare MAP-Elites in Algorithm~\ref{algo:map-elites} with $(\mu+1)$-EA in Algorithm~\ref{algo:ea}. As the problems studied in Sections~\ref{sec-submodular} and~\ref{sec-setcover} are pseudo-Boolean problems, i.e., the solution space $\mathcal{X} = \{0,1\}^n$ (where $n$ denotes the problem size), we employ the commonly used \emph{bit-wise mutation} operator in Algorithms~\ref{algo:map-elites} and~\ref{algo:ea}, which flips each bit of a solution $\bm x$ with probability $1/n$ to generate a new solution $\bm{x}'$. For fairness, the population size $\mu$ of $(\mu+1)$-EA is set to the number of cells in the archive $M$; as $(\mu+1)$-EA generates $\mu$ initial solutions randomly, the number of randomly generated initial solutions (i.e., the parameter $I$) of MAP-Elites is set equal to $\mu$. These settings are to make that the performance gap between MAP-Elites and $(\mu+1)$-EA is purely due to the mechanism of finding a set of high-quality and diverse solutions of MAP-Elites.

To be consistent with the original setting of MAP-Elites in~\cite{mouret2015illuminating}, we have used strict survivor selection, i.e., $f(\bm{x}')>f(M(\bm{b}(\bm{x}')))$ in line~10 of Algorithm~\ref{algo:map-elites} and $f(\bm{x}')>f(\bm{y})$ in line~6 of Algorithm~\ref{algo:ea}. However, our analyses in the following two sections will not be affected if allowing accepting the same good solutions, i.e., replacing $>$ with $\geq$. We consider maximization problems in Algorithms~\ref{algo:map-elites} and~\ref{algo:ea}. For minimization, we can obtain the corresponding algorithms by slight modification, i.e., replacing $>$ with $<$ in both algorithms, and changing $\bm{y}=\arg\min_{\bm{y}\in P}f(\bm{y})$ in line~5 of Algorithm~\ref{algo:ea} to $\bm{y}=\arg\max_{\bm{y}\in P}f(\bm{y})$.

To measure the performance of QD algorithms, the three metrics, i.e., \emph{optimization}, \emph{coverage} and \emph{QD-Score}, are often used~\cite{qd-frontier,qd-survey-framework}. Optimization corresponds to the largest objective value of solutions in the archive, i.e., the objective value of the highest-performing solution found by the algorithm. Coverage is the number of non-empty cells, i.e., the total number of solutions, in the archive. QD-Score is equal to the total sum of objective values across all solutions in the archive, which reflects both the quality and diversity of the found solutions. 

In this paper, we only focus on the optimization metric, serving for our purpose of showing that QD algorithms can be helpful for optimization, i.e., finding better objective values. Specifically, we will analyze the expected running time complexity of MAP-Elites or $(\mu+1)$-EA, which counts the number of objective function evaluations (the most time-consuming step in the evolutionary process) until finding an $\alpha$-approximate solution, and has been a leading theoretical aspect for randomized search heuristics~\cite{neumann.witt.10,auger.doerr.11,zhou2019evolutionary,doerr2020theory}. For maximization, an $\alpha$-approximate solution $\bm{x}$ means $f(\bm{x}) \geq \alpha \cdot \mathrm{OPT}$, where $\alpha \leq 1$ is the approximation ratio, and $\mathrm{OPT}$ is the optimal function value. For minimization, it means $f(\bm{x}) \leq \alpha \cdot \mathrm{OPT}$, where $\alpha \geq 1$. Because both MAP-Elites and $(\mu+1)$-EA perform only one objective evaluation (i.e., the evaluation $f(\bm{x}')$ of the generated offspring solution $\bm{x}'$) in each iteration, the expected running time equals the number (i.e., $\mu$) of initial solution evaluations plus the expected number of iterations.

\section{Monotone Approximately Submodular Maximization with a Size Constraint}\label{sec-submodular}

Let $\mathbb{R}$ and $\mathbb{R}^{+}$ denote the set of reals and non-negative reals, respectively. Given a finite non-empty set $V=\{v_1,\ldots,v_n\}$, a set function $f:2^V \rightarrow \mathbb{R}$ is defined on subsets of $V$. $f$ is monotone if $\forall X \subseteq Y$, $f(X) \leq f(Y)$. Without loss of generality, we assume that monotone functions are normalized, i.e., $f(\emptyset)=0$. A set function $f$ is submodular~\cite{nemhauser1978analysis} if for any $X \subseteq Y \subseteq V$ and $v \notin Y$,
\begin{align}\label{def-submodular-1}
f(X \cup \{v\})-f(X) \geq f(Y \cup \{v\}) - f(Y);
\end{align}
or equivalently for any $X \subseteq Y \subseteq V$,
\begin{align}\label{def-submodular-2}
f(Y)-f(X) \leq \sum\nolimits_{v \in Y \setminus X} \big(f(X \cup \{v\})-f(X)\big).
\end{align}
Eq.~(\refeq{def-submodular-1}) intuitively represents the ``diminishing returns" property, i.e., adding an item to a set gives a larger benefit than adding the same item to its superset. Based on Eq.~(\refeq{def-submodular-2}), the submodularity ratio as presented in Definition~\ref{def-approx-submodular-2} is introduced to measure how close a general set function $f$ is to submodularity. When $f$ is monotone, it holds: $\forall X, l:  0 \leq \gamma_{X,l}(f) \leq 1$; $f$ is submodular if and only if $\forall X, l: \gamma_{X,l}(f)=1$~\cite{qian2018approximation}.

\begin{definition}[Submodularity Ratio~\cite{das2011submodular}]\label{def-approx-submodular-2}
Let $f: 2^V \rightarrow \mathbb{R}$ be a set function. The submodularity ratio of $f$ with respect to a set $X \subseteq V$ and a parameter $l \geq 1$ is
$$
\gamma_{X,l}(f)=\min_{L \subseteq X, S: |S|\leq l, S \cap L =\emptyset} \frac{\sum_{v \in S} (f(L \cup \{v\})-f(L))}{f(L \cup S)-f(L)}.
$$
\end{definition}

As presented in Definition~\ref{def-prob-nonsubmodular}, the studied problem is to find a subset with at most $k$ items such that a given monotone and approximately submodular objective function is maximized. It is NP-hard, and has many applications, such as maximum coverage~\cite{feige1998threshold}, influence maximization~\cite{kempe2003maximizing}, sensor placement~\cite{krause2008near}, sparse regression~\cite{das2011submodular}, unsupervised feature selection~\cite{feng2019unsupervised}, and human assisted ridge regression~\cite{liu2023human}, just to name a few.

\begin{definition}[Monotone Approximately Submodular Function Maximization with a Size Constraint]\label{def-prob-nonsubmodular}
Given a monotone and approximately submodular function $f: 2^V \rightarrow \mathbb{R}^+$ and a budget $k$, the goal is to find a subset $X\subseteq V$ such that
$$
\arg \max\nolimits_{X \subseteq V}\;\; f(X)\quad  \text{s.t.} \quad |X| \leq k.
$$
\end{definition}

In the following analysis, we equivalently reformulate the above problem as an unconstrained maximization problem by setting $f(X)=-1$ when $|X|>k$, i.e., $X$ does not satisfy the size constraint. Note that a subset $X$ of $V$ can be naturally represented by a Boolean vector $\bm{x} \in \{0,1\}^n$, where the $i$-th bit $x_i=1$ if $v_i \in X$, and $x_i=0$ otherwise. Throughout the paper, we will not distinguish $\bm{x}\in \{0,1\}^n$ and its corresponding subset $\{v_i \in V \mid x_i=1\}$ for convenience.

\subsection{Analysis of MAP-Elites}

To apply MAP-Elites in Algorithm~\ref{algo:map-elites} to maximize monotone approximately submodular functions with a size constraint, we use the number of 1-bits (denoted as $|\bm{x}|$, which equals $\sum^{n}_{i=1} x_i$) of a solution $\bm{x}\in \{0,1\}^n$ as the behavior descriptor. Thus, the archive $M$ maintained by MAP-Elites contains $n+1$ cells, where for any $0\leq i\leq n$, the $i$-th cell stores the best found solution with $i$ 1-bits. In each iteration of MAP-Elites, if an offspring solution $\bm{x}'$ with $j$ 1-bits is generated, the $j$-th cell in the archive $M$ will be examined: if the cell is empty or the existing solution in the cell is worse than $\bm{x}'$, the $j$-th cell will then be occupied by $\bm{x}'$.

We prove in Theorem~\ref{theo-nonsubmodular} that after running $O(n^2(\log n+k))$ expected time, MAP-Elites achieves an approximation ratio of $1-e^{-\gamma_{\min}}$, which reaches the optimal polynomial-time approximation ratio~\cite{harshaw2019submodular}. Inspired from the analysis of GSEMO (a multi-objective EA with similar procedure to MAP-Elites)~\cite{friedrich2015maximizing,qian.nips15,qian2019maximizing}, the proof can be accomplished by following the behavior of the greedy algorithm~\cite{das2011submodular}. Note that the parameter $I=n+1$ implies that MAP-Elites generates $n+1$ initial solutions randomly. The value of $I$ does not affect the analysis, and such a setting is just for fairness, because the population size $\mu$ of $(\mu+1)$-EA will be set to the number of cells (i.e., $n+1$) contained by the archive $M$. 

\begin{theorem}\label{theo-nonsubmodular}
For maximizing a monotone approximately submodular function $f$ with a size constraint $k$, the expected running time of MAP-Elites with the parameter $I=n+1$, until finding a solution $\bm{x}$ with $|\bm{x}| = k$ and $$f(\bm{x}) \geq (1-e^{-\gamma_{\min}}) \cdot \mathrm{OPT},$$ is $O(n^2(\log n+k))$, where $\gamma_{\min}=\min_{\bm{x}:|\bm{x}|=k-1}\gamma_{\bm{x},k}$, $\gamma_{\bm{x},k}$ is the submodularity ratio of $f$ w.r.t. $\bm{x}$ and $k$ as in Definition~\ref{def-approx-submodular-2}, and $\mathrm{OPT}$ denotes the optimal function value.
\end{theorem}

The proof relies on Lemma~\ref{lemma-submodular}, which shows that any subset $\bm{x} \in \{0,1\}^n$ can be improved by adding a specific item, such that the increment on $f$ is roughly proportional to $\mathrm{OPT}-f(\bm{x})$ and depends on the submodularity ratio $\gamma_{\bm{x},k}$.

\begin{lemma}[Lemma~1 of \cite{qian2016parallel}]\label{lemma-submodular}
Given a monotone approximately submodular function $f$, $\forall \bm{x} \in \{0,1\}^n$, there exists one item $v \notin \bm{x}$ such that
\begin{align}\label{eq-mid-6}
&f(\bm{x} \cup \{v\})-f(\bm{x}) \geq \frac{\gamma_{\bm{x},k}}{k} \cdot (\mathrm{OPT}-f(\bm{x})).
\end{align}
\end{lemma}

\begin{myproof}{Theorem~\ref{theo-nonsubmodular}}
We divide the optimization process of MAP-Elites into two phases: (1) starts after initialization and finishes after finding the special solution $\bm{0}$ (i.e., the empty set); (2) starts after phase (1) and finishes after finding a solution with the desired approximation guarantee. Note that the archive $M$ will always contain at most $n+1$ solutions, because there are $n+1$ cells and the $i$-th cell only stores the best found solution with $i$ 1-bits. We will use $|M|$ to denote the number of solutions in the archive $M$, satisfying $|M|\leq n+1$.

For phase (1), we consider the minimum number of 1-bits of the solutions in the archive $M$, denoted by $J_{\min}$. That is, $J_{\min}=\min\{|\bm{x}| \mid \bm{x} \in M\}$. Assume that currently $J_{\min}=i>0$, and let $\bm{x}$ be the corresponding solution, i.e., $|\bm{x}|=i$. $J_{\min}$ cannot increase because $\bm{x}$ can only be replaced by a better solution (having a larger $f$ value) with the same number of 1-bits, i.e., belonging to the same cell. In each iteration of MAP-Elites, to decrease $J_{\min}$, it is sufficient to select $\bm{x}$ in line~7 of Algorithm~\ref{algo:map-elites} and flip only one 1-bit of $\bm{x}$ in line~8. This is because the newly generated solution $\bm{x}'$ now has $(i-1)$ 1-bits (i.e., $|\bm{x}'|=i-1$), and will enter into the $(i-1)$-th cell, which must be empty. The probability of selecting $\bm{x}$ in line~7 of Algorithm~\ref{algo:map-elites} is $1/|M| \geq 1/(n+1)$ due to uniform selection, and the probability of flipping only one 1-bit of $\bm{x}$ in line~8 is $(i/n)(1-1/n)^{n-1} \geq i/(en)$, since $\bm{x}$ has $i$ 1-bits and $(1-1/n)^{n-1}\geq 1/e$. Thus, the probability of decreasing $J_{\min}$ by 1 in each iteration of MAP-Elites is at least $i/(en(n+1))$. Because $J_{\min}$ is at most $n$, the expected number of iterations required by phase (1) (i.e., making $J_{\min}=0$) is at most $\sum^{n}_{i=1} en(n+1)/i=O(n^2\log n)$.

For phase (2), we consider a quantity $J_{\max}$, which is defined as the maximum value of $j \in \{0,1,\ldots,k\}$ such that in the archive $M$, there exists a solution $\bm{x}$ with $|\bm{x}| =j$ and $f(\bm{x}) \geq (1-(1-\gamma_{\min}/k)^j) \cdot \mathrm{OPT}$. We only need to analyze the expected number of iterations until $J_{\max}=k$, which implies that there exists one solution $\bm{x}$ in $M$ satisfying that $|\bm{x}| = k$ and $f(\bm{x}) \geq (1-(1-\gamma_{\min}/k)^k) \cdot \mathrm{OPT} \geq (1-e^{-\gamma_{\min}}) \cdot \mathrm{OPT}$. That is, the desired approximation guarantee is reached.

The current value of $J_{\max}$ is at least 0, since the archive $M$ contains the solution $\bm{0}$, satisfying that $|\bm{0}|=0$ and $f(\bm{0})=0$. Assume that currently $J_{\max}=i <k$, and let $\bm{x}$ be the corresponding solution, i.e., $|\bm{x}|= i$ and $f(\bm{x})\geq (1-(1-\gamma_{\min}/k)^i) \cdot \mathrm{OPT}$. $J_{\max}$ cannot decrease because $\bm{x}$ can only be replaced by a newly generated solution $\bm{x}'$ satisfying that $|\bm{x}'|=i$ and $f(\bm{x}') > f(\bm{x})$, as shown in lines~10--12 of Algorithm~\ref{algo:map-elites}. By Lemma~\ref{lemma-submodular}, we know that flipping one specific 0-bit of $\bm{x}$ (i.e., adding a specific item) can generate a new solution $\bm{x}'$, satisfying  $f(\bm{x}')-f(\bm{x}) \geq (\gamma_{\bm{x},k}/k) \cdot (\mathrm{OPT}-f(\bm{x}))$. Then, we have
\begin{align}
f(\bm{x}') &\geq \left(1-\gamma_{\bm{x},k}/k\right)\cdot f(\bm{x})+(\gamma_{\bm{x},k}/k)\cdot \mathrm{OPT} \\
&\geq \left(1-\left(1-\gamma_{\bm{x},k}/k\right)\left(1-\gamma_{\min}/k\right)^{i}\right)\cdot \mathrm{OPT}\\
&\geq \left(1-\left(1-\gamma_{\min}/k\right)^{i+1}\right)\cdot \mathrm{OPT},
\end{align}
where the second inequality holds by $f(\bm{x})\geq (1-(1-\gamma_{\min}/k)^i) \cdot \mathrm{OPT}$, and the last inequality holds by $\gamma_{\bm{x},k} \geq \gamma_{\min}=\min_{\bm{z}:|\bm{z}|=k-1}\gamma_{\bm{z},k}$, which can be derived from $|\bm{x}|=i<k$ and $\gamma_{\bm{z},k}$ decreasing with $\bm{z}$. Since $|\bm{x}'|=|\bm{x}|+1 = i+1$, $\bm{x}'$ will be included into the archive $M$; otherwise, the $(i+1)$-th cell must already contain a better solution $\bm{y}$ satisfying that $|\bm{y}|=i+1$ and $f(\bm{y})\geq f(\bm{x}')$, and this implies that $J_{\max}$ has already been larger than $i$, contradicting with the assumption $J_{\max}=i$. After including $\bm{x}'$, $J_{\max} = i+1$. Thus, $J_{\max}$ can increase by 1 in one iteration with probability at least $(1/|M|) \cdot (1/n)(1-1/n)^{n-1} \geq 1/(en(n+1))$, where $1/|M|\geq 1/(n+1)$ is a lower bound on the probability of selecting $\bm{x}$ in line~7 of Algorithm~\ref{algo:map-elites}, and $(1/n)(1-1/n)^{n-1}$ is the probability of flipping a specific bit of $\bm{x}$ while keeping other bits unchanged in line~8. This implies that it needs at most $en(n+1)$ expected number of iterations to increase $J_{\max}$. Thus, after at most $k \cdot en(n+1)$ iterations in expectation, $J_{\max}$ must have reached $k$.

By summing up the expected running time of the above two phases, we get that the expected running time of MAP-Elites for finding a solution $\bm{x}$ with $|\bm{x}|= k$ and $f(\bm{x})\geq (1-e^{-\gamma_{\min}}) \cdot \mathrm{OPT}$ is $O(n^2\log n +kn(n+1))=O(n^2(\log n+k))$. The additional evaluations for the $n+1$ initial solutions will not affect this upper bound. Thus, the theorem holds.\vspace{0.3em} 
\end{myproof}

When $f$ is submodular, $\forall \bm{x}: \gamma_{\bm{x},k}=1$, implying $\gamma_{\min}=1$. Thus, we have
\begin{corollary}\label{corollary-submodular}
For maximizing a monotone submodular function $f$ with a size constraint $k$, the expected running time of MAP-Elites with the parameter $I=n+1$, until finding a solution $\bm{x}$ with $|\bm{x}| = k$ and $$f(\bm{x}) \geq (1-1/e) \cdot \mathrm{OPT},$$ is $O(n^2(\log n+k))$.
\end{corollary}

The result in Corollary~\ref{corollary-submodular} is consistent with Theorem~5.1 in~\cite{bossek2023runtime}. Note that the MAP-Elites algorithm analyzed in~\cite{bossek2023runtime} is slightly different from ours, which generates only one random initial solution (i.e., the parameter $I=1$) and replaces the solution in a cell if the generated offspring solution belonging to the same cell is not worse (i.e., $f(\bm{x}')>f(M(\bm{b}(\bm{x}')))$ in line~10 of Algorithm~\ref{algo:map-elites} changes to $f(\bm{x}')\geq f(M(\bm{b}(\bm{x}')))$). We can find that these slight differences do not affect the result.

\subsection{Analysis of $(\mu+1)$-EA}

To apply $(\mu+1)$-EA in Algorithm~\ref{algo:ea} to maximize monotone approximately submodular functions with a size constraint, we set the population size $\mu$ to the number of cells contained by the archive $M$ of MAP-Elites, which is $n+1$. As we have mentioned in Section~\ref{subsec-algo-setting}, such a setting is for fairness, which makes the population size of $(\mu+1)$-EA equal to the largest archive size of MAP-Elites. In each iteration of $(\mu+1)$-EA, the generated offspring solution $\bm{x}'$ will be used to replace the worst solution in the current population if $\bm{x}'$ is better.

Next, we will give a concrete instance of the classic NP-hard problem, maximum coverage~\cite{feige1998threshold} (which is an application of monotone submodular maximization with a size constraint), where $(\mu+1)$-EA fails to achieve a good approximation ratio in polynomial time. Given a family $V$ of sets that cover a universe $U$ of elements, the goal of maximum coverage as presented in Definition~\ref{def_mc} is to select at most $k$ sets from $V$ to make their union maximal. The objective function $f(\bm{x})=|\bigcup_{i: x_i=1}S_i|$ is monotone and submodular.

\begin{definition}[Maximum Coverage]\label{def_mc}
Given a ground set $U$, a collection $V=\{S_1,S_2,\ldots,S_n\}$ of subsets of $U$, and a budget $k$, the goal is to find a subset of $V$ (represented by $\bm{x} \in \{0,1\}^n$) such that
\begin{align}
\mathop{\arg\max}\nolimits_{\bm{x} \in \{0,1\}^n}  \;\; |\bigcup\nolimits_{i: x_i=1} S_i|  \quad \text{s.t.}  \quad |\bm{x}| \leq k.
\end{align}   
\end{definition}

The considered instance of maximum coverage as presented in Example~\ref{example-mc} is from~\cite{friedrich2010approximating,friedrich2015maximizing}, and can be obtained from the complete bipartite graph in Figure~\ref{fig_bg}(a). The ground set $U$ is given by the set of edges. Each vertex $v_i$ corresponds to one subset $S_i$ of $U$, which consists of the edges adjacent to the vertex $v_i$. The goal is to select at most $k=(1+\delta)n/3$ subsets of $U$ to make the number of contained edges maximal. 

\begin{example}\label{example-mc}
The ground set $U$ contains all the edges of the complete bipartite graph in Figure~\ref{fig_bg}(a). Each set $S_i$ in $V$ consists of the edges adjacent to the vertex $v_i$, i.e., $\forall 1 \leq i \leq (1+\delta)n/3$, $S_i=\{(v_i,v_{(1+\delta)n/3+1}),\ldots,(v_i,v_n)\}$; $\forall (1+\delta)n/3+1 \leq i \leq n$, $S_i=\{(v_i,v_1),\ldots,(v_i,v_{(1+\delta)n/3})\}$, where $(v_i,v_j)$ denotes the edge between the vertexes $v_i$ and $v_j$. The budget $k=(1+\delta)n/3$, and $\delta$ is a small positive constant close to 0.
\end{example}

The two parts of the complete bipartite graph in Figure~\ref{fig_bg}(a) contain $(1+\delta)n/3$ and $(2-\delta)n/3$ vertexes, respectively. As the budget $k=(1+\delta)n/3<(2-\delta)n/3$, the unique optimal solution is $\bm{x}^*=1^{(1+\delta)n/3}0^{(2-\delta)n/3}$ (i.e., $\{S_1,\ldots,S_{(1+\delta)n/3}\}$), which covers all the edges, and has the objective value $f(\bm{x}^*)=((1+\delta)n/3) \cdot ((2-\delta)n/3)=(1+\delta)(2-\delta)n^2/9$. We can also observe that if a solution $$\bm{x}_{\mathrm{local}} \!\in\! \{0^{(1+\delta)n/3}\bm{y} \mid \bm{y} \!\in\! \{0,1\}^{(2-\delta)n/3} \wedge |\bm{y}|\!=\!(1+\delta)n/3\},$$ i.e., contains $(1+\delta)n/3$ sets from $\{S_{(1+\delta)n/3+1},\ldots,S_n\}$, it is local optimal, with the objective value $f(\bm{x}_{\mathrm{local}})=((1+\delta)n/3) \cdot ((1+\delta)n/3)=(1+\delta)^2 n^2/9$. Adding a set from $\{S_1,\ldots,S_{(1+\delta)n/3}\}$ into $\bm{x}_{\mathrm{local}}$ will cover $(2-\delta)n/3-(1+\delta)n/3=(1-2\delta)n/3$ more edges. If $i$ sets from $\{S_1,\ldots,S_{(1+\delta)n/3}\}$ have been added into $\bm{x}_{\mathrm{local}}$, then deleting one set from $\bm{x}_{\mathrm{local}}$ will decrease the number of edges by $(1+\delta)n/3-i$. To make $(1+\delta)n/3-i \leq (1-2\delta)n/3$, $i$ should be at least $\delta n$. Thus, to improve the local optimal solution $\bm{x}_{\mathrm{local}}$, at least $\delta n$ sets from $\{S_1,\ldots,S_{(1+\delta)n/3}\}$ need to be added into $\bm{x}_{\mathrm{local}}$, and meanwhile at least $\delta n$ sets from $\bm{x}_{\mathrm{local}}$ have to be deleted. The approximation ratio of $\bm{x}_{\mathrm{local}}$ is $f(\bm{x}_{\mathrm{local}})/f(\bm{x}^*)=(1+\delta)/(2-\delta)$.

\begin{figure}[t!]\centering
\begin{minipage}[c]{0.5\linewidth}\centering
  \includegraphics[width=\linewidth]{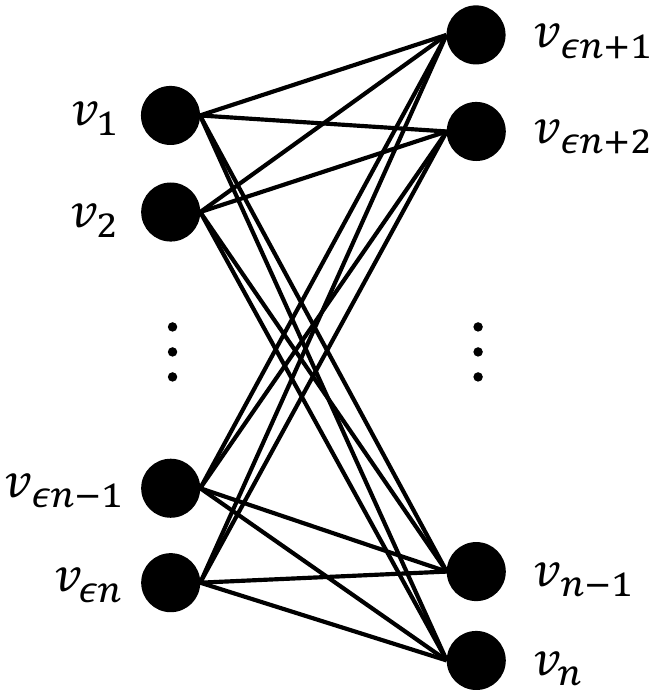}
\end{minipage}
\begin{minipage}[c]{0.44\linewidth}\centering
  \includegraphics[width=\linewidth]{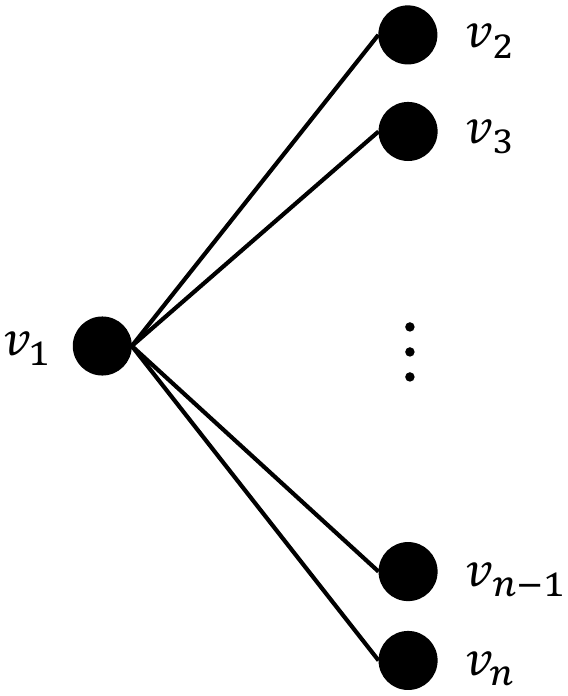}
\end{minipage}\\\vspace{0.5em}
\begin{minipage}[c]{0.5\linewidth}\centering
    \small(a) $\epsilon=(1+\delta)/3$
\end{minipage}
\begin{minipage}[c]{0.44\linewidth}\centering
    \small(b)
\end{minipage}
\caption{Two complete bipartite graphs.}\label{fig_bg}
\end{figure}

We prove in Theorem~\ref{theo_ea_mc} that on the maximum coverage instance in Example~\ref{example-mc}, $(\mu+1)$-EA requires at least exponential (w.r.t. $n$) expected running time to achieve an approximation ratio larger than $(1+\delta)/(2-\delta)$. The main proof idea is to show that there is some probability that the population of $(\mu+1)$-EA will get trapped in a local optimal solution $\bm{x}_{\mathrm{local}}$, escaping from which requires at least $2\delta n$ bits to be flipped simultaneously.

\begin{theorem}\label{theo_ea_mc}
There exists an instance of monotone submodular maximization with a size constraint (more precisely, the maximum coverage instance in Example~\ref{example-mc}), where the expected running time of $(\mu+1)$-EA with $\mu=n+1$ for achieving an approximation ratio larger than $(1+\delta)/(2-\delta)$ is at least exponential w.r.t. $n$, where $\delta$ is a small positive constant close to 0.
\end{theorem}
\begin{proof}
In line~1 of Algorithm~\ref{algo:ea}, $(\mu+1)$-EA generates $n+1$ initial solutions randomly, where each solution is sampled from $\{0,1\}^n$ uniformly at random. We first consider the event (denoted as $A$) that among the generated $n+1$ initial solutions, one is a local optimal solution $\bm{x}_{\mathrm{local}}$, and the other $n$ solutions do not satisfy the size constraint, i.e., contain more than $k=(1+\delta)n/3$ 1-bits. For a solution $\bm{x}$ randomly sampled from $\{0,1\}^n$, the probability of being local optimal is $P(\bm{x} \in \{0^{(1+\delta)n/3}\bm{y} \mid \bm{y} \in \{0,1\}^{(2-\delta)n/3} \wedge |\bm{y}|=(1+\delta)n/3\})=\binom{(2-\delta)n/3}{(1+\delta)n/3}/2^n\geq 1/2^n$; the probability of having more than $(1+\delta)n/3$ 1-bits is $P(|\bm{x}|>(1+\delta)n/3)=1-P(|\bm{x}|\leq (1+\delta)n/3)=1-P(|\bm{x}|\leq (1-(1-2\delta)/3)\cdot (n/2))\geq 1-e^{-\frac{n}{4}(\frac{1-2\delta}{3})^2}$, where the inequality holds by the Chernoff bound. Thus, the event $A$ happens with probability 
$$
P(A)\geq \frac{n+1}{2^n}\cdot \left(1-e^{-\frac{n}{4}(\frac{1-2\delta}{3})^2}\right)^n = \frac{n+1}{2^n}(1-o(1)),
$$
where the factor $n+1$ is because $\bm{x}_{\mathrm{local}}$ can be any of the $(n+1)$ initial solutions.

Given the event $A$, we consider that in the first $n$ iterations of $(\mu+1)$-EA, a local optimal solution is always selected from the current population in line~3 of Algorithm~\ref{algo:ea}, and no bits are flipped by the bit-wise mutation operator in line~4. Such an event is denoted as $B$. Because the objective value of a solution which does not satisfy the size constraint is set to $-1$, the copied local optimal solution will replace one solution violating the size constraint in the population after each iteration according to the population updating procedure in lines~5--8 of Algorithm~\ref{algo:ea}. Thus, if the event $B$ happens, the population of $(\mu+1)$-EA will be full of local optimal solutions. Because uniform selection is performed in line~3 of Algorithm~\ref{algo:ea} and the number of local optimal solutions increases by 1 after each iteration, the event $B$ happens with probability
\begin{align}
P(B)&=\prod^{n}_{i=1}\left(\frac{i}{n+1}\left(1\!-\!\frac{1}{n}\right)^n\right)= \frac{n!}{(n\!+\!1)^n}\left(1\!-\!\frac{1}{n}\right)^{n^2}\\
&\geq \sqrt{2\pi n}\left(\frac{n}{e}\right)^n\cdot \frac{1}{(n+1)^n}\cdot \frac{1}{(2e)^n} \geq \frac{\sqrt{2\pi n}}{e(2e^2)^n},
\end{align}
where $i/(n+1)$ is the probability of selecting a local optimal solution from the population in the $i$-th iteration, $(1-1/n)^n$ is the probability of keeping all the $n$ bits of the selected solution unchanged in mutation, the first inequality holds by $n!\geq \sqrt{2\pi n}(n/e)^n$ due to Stirling's formula and $(1-1/n)^n\geq 1/(2e)$, and the last inequality holds by $(n/(n+1))^n=(1-1/(n+1))^n\geq 1/e$. 

After the events $A$ and $B$ happen, the population consists of $n+1$ local optimal solutions from $\{0^{(1+\delta)n/3}\bm{y} \mid \bm{y} \in \{0,1\}^{(2-\delta)n/3} \wedge |\bm{y}|=(1+\delta)n/3\}$. As we have analyzed before, to improve a local optimal solution $\bm{x}_{\mathrm{local}}$, at least $\delta n$ sets from $\{S_1,\ldots,S_{(1+\delta)n/3}\}$ need to be added into $\bm{x}_{\mathrm{local}}$, and meanwhile at least $\delta n$ sets from $\bm{x}_{\mathrm{local}}$ have to be deleted. That is, when mutating $\bm{x}_{\mathrm{local}}$, it needs to flip at least $\delta n$ of its $(1+\delta)n/3$ leading 0-bits and at least $\delta n$ of its $(1+\delta)n/3$ 1-bits, the probability of which is at most
\begin{align}
&\binom{(1+\delta)n/3}{\delta n} \cdot \binom{(1+\delta)n/3}{\delta n} \cdot \left(\frac{1}{n}\right)^{2\delta n}\\
&\leq \frac{((1+\delta)n/3)^{2\delta n}}{(\delta n)!^2}\cdot \!\left(\frac{1}{n}\right)^{2\delta n} \!\!\leq \frac{1}{2\pi\delta n}\cdot \!\left(\frac{e(1+\delta)}{3\delta n}\right)^{2\delta n},
\end{align}
where the last inequality holds by $(\delta n)!\geq \sqrt{2\pi \delta n}(\delta n/e)^{\delta n}$ due to Stirling's formula. Thus, given the events $A$ and $B$, the expected running time to generate a solution better than $\bm{x}_{\mathrm{local}}$ is at least $2\pi\delta n(3\delta n/(e(1+\delta)))^{2\delta n}$.

Because the approximation ratio of $\bm{x}_{\mathrm{local}}$ is $(1+\delta)/(2-\delta)$, $P(A) \!\geq (n+1)(1-o(1))/2^n$ and $P(B)\geq \sqrt{2\pi n}/(e(2e^2)^n)$, the expected running time of $(\mu+1)$-EA for achieving an approximation ratio larger than $(1+\delta)/(2-\delta)$ is at least 
$$
\frac{n+1}{2^n}(1-o(1))\cdot \frac{\sqrt{2\pi n}}{e(2e^2)^n} \cdot 2\pi\delta n\left(\frac{3\delta n}{e(1+\delta)}\right)^{2\delta n},
$$
which is exponential w.r.t. $n$ because $\delta$ is a small positive constant close to 0. Thus, the theorem holds.
\end{proof}

The comparison between Corollary~\ref{corollary-submodular} and Theorem~\ref{theo_ea_mc} clearly shows the advantage of MAP-Elites over $(\mu+1)$-EA. That is, for monotone submodular maximization with a size constraint, MAP-Elites can always achieve the $(1-1/e)$-approximation ratio in expected polynomial time, while there exists an instance where $(\mu+1)$-EA requires at least expected exponential time to achieve an approximation ratio of nearly $1/2$. From the proofs, we can find that MAP-Elites explicitly maintains multiple cells of the behavior space, allowing it to achieve a good solution with $k$ 1-bits by following a path across different cells (i.e., gradually adding the number of 1-bits), while $(\mu+1)$-EA only focuses on the objective value, and may lose diversity in the population and get trapped in local optima (with $k$ 1-bits). This discloses the benefit of simultaneously searching for high-performing solutions at diverse regions in the behavior space, which may provide stepping stones to high-performing solutions in a region that may be difficult to be discovered if searching only in that region.

\section{Set Cover}\label{sec-setcover}

Set cover~\cite{feige1998threshold} as presented in Definition~\ref{def_sc} is a well-known NP-hard combinatorial optimization problem. Given a family $V$ of sets (where each set has a corresponding positive weight) that cover a universe $U$ of elements, the goal is to find a subset of $V$ with the minimum weight such that all the elements of $U$ are covered.

\begin{definition}[Set Cover]\label{def_sc}
Given a ground set $U=\{e_1,e_2,\ldots,e_m\}$, and a collection $V=\{S_1,S_2,\ldots,S_n\}$ of subsets of $U$ with corresponding weights $w: V \rightarrow \mathbb{R}^+$, the goal is to find a subset of $V$ (represented by $\bm{x} \in \{0,1\}^n$) such that 
\begin{align}
\mathop{\arg\min}\nolimits_{x \in \{0,1\}^n}  \;\;  \sum\nolimits_{i=1}^{n} w_ix_i \quad \text{s.t.}  \quad \bigcup\nolimits_{i: x_i=1} S_i=U.
\end{align}    
\end{definition}

Note that a Boolean vector $\bm{x} \in \{0,1\}^n$ represent a subset $X$ of $V$, where the $i$-th bit $x_i=1$ if $S_i \in X$, and $x_i=0$ otherwise. Let $w_{\mathrm{\max}}=\max\{w_i \mid 1\leq i\leq n\}$ denote the maximum weight of a single subset $S_i$. In the following analysis, we equivalently reformulate the above problem as an unconstrained minimization problem by setting \begin{align}\label{eq-sc}\forall \bm{x}\in\{0,1\}^n: f(\bm{x})=w(\bm{x})+\lambda\cdot (m-c(\bm{x})),\end{align}
where $w(\bm{x})=\sum\nolimits_{i=1}^{n} w_ix_i$ denotes the weight of $\bm{x}$, $c(\bm{x})=|\bigcup\nolimits_{i: x_i=1} S_i|$ denotes the number of elements covered by $\bm{x}$, and $\lambda>nw_{\mathrm{\max}}$ is a sufficiently large penalty. That is, a solution covering more elements is better; if covering the same number of elements, a solution with a smaller weight is better.

\subsection{Analysis of MAP-Elites}

To apply MAP-Elites in Algorithm~\ref{algo:map-elites} to solve the set cover problem, we use the number $c(\bm{x})=|\bigcup\nolimits_{i: x_i=1} S_i|$ of covered elements of a solution $\bm{x}\in \{0,1\}^n$ as the behavior descriptor. Thus, the archive $M$ maintained by MAP-Elites contains $m+1$ cells, where for any $0\leq i\leq m$, the $i$-th cell stores the best found solution covering $i$ elements. In each iteration of MAP-Elites, if an offspring solution $\bm{x}'$ with $c(\bm{x}')=j$ is generated, the $j$-th cell in the archive $M$ will be examined: if the cell is empty or the existing solution in the cell is worse than $\bm{x}'$ (i.e., has a larger weight than $\bm{x}'$), the $j$-th cell will then be occupied by $\bm{x}'$.

We prove in Theorem~\ref{theo-sc} that after running $O(mn(m+\log n + \log (w_{\max}/w_{\min})))$ expected time, MAP-Elites achieves an approximation ratio of $\ln m+1$, which has been optimal up to a constant factor, unless P$\,=\,$NP~\cite{feige1998threshold}. The proof is accomplished by following the behavior of the greedy algorithm~\cite{chvatal1979greedy}, inspired from the analysis of GSEMO~\cite{friedrich2010approximating,qian2015constrained}. The parameter $I$ of MAP-Elites is set to $m+1$, i.e., it generates $m+1$ initial solutions randomly. As for solving the problem of monotone approximately submodular maximization with a size constraint in the last section, such a setting is for fairness, because the population size $\mu$ of $(\mu+1)$-EA will be set to the number of cells (i.e., $m+1$) contained by the archive $M$. 

\begin{theorem}\label{theo-sc}
For the set cover problem, the expected running time of MAP-Elites with the parameter $I=m+1$, until finding a solution $\bm{x}$ with $c(\bm{x}) = m$ and $$f(\bm{x}) \leq (\ln m +1) \cdot \mathrm{OPT},$$ is $O(mn(m+\log n + \log (w_{\max}/w_{\min})))$, where $m$ is the size of the ground set $U$, $c(\bm{x})=|\bigcup\nolimits_{i: x_i=1} S_i|$ denotes the number of elements covered by $\bm{x}$, $\mathrm{OPT}$ denotes the optimal function value, and $w_{\max}=\max\{w_i\mid 1\leq i\leq n\}$ and $w_{\min}=\min\{w_i\mid 1\leq i\leq n\}$ denote the maximum and minimum weights of a single set $S_i$, respectively.
\end{theorem}
\begin{proof}
We divide the optimization process of MAP-Elites into two phases: (1) starts after initialization and finishes after finding the special solution $\bm{0}$ which does not cover any element; (2) starts after phase (1) and finishes after finding a solution with the desired approximation guarantee. Note that the archive $M$ will always contain at most $m+1$ solutions, i.e., $|M| \leq m+1$, because there are $m+1$ cells and the $i$-th cell only stores the best found solution covering $i$ elements.

For phase (1), we consider the minimum weight of solutions in the archive $M$ after running $t$ iterations, denoted as $X_t=\min\{w(\bm{x}) \mid \bm{x} \in M \}$. Note that $X_t=0$ implies that the solution $\bm{0}$ has been found. Then, we are to analyze the expected change of $X_t$ after one iteration, denoted as $E(X_t-X_{t+1} \mid X_t)$. Let $\bm{x}$ be a corresponding solution with $w(\bm{x})=X_t$. We can observe that $X_t$ cannot increase (i.e., $X_{t+1} \leq X_t$) because $\bm{x}$ can only be replaced by a solution which covers the same number of elements as $\bm{x}$ (i.e., belongs to the same cell as $\bm{x}$) and has a smaller weight. In the $(t+1)$-th iteration, we consider that $\bm{x}$ is selected in line~7 of Algorithm~\ref{algo:map-elites}, occurring with probability $1/|M| \geq 1/(m+1)$ due to uniform selection and $|M| \leq m+1$. Then, in line~8, let $\bm{x}^i$ denote the solution generated by flipping the $i$-th bit (i.e., $x_i$) of $\bm{x}$ and keeping other bits unchanged, which happens with probability $(1/n)(1-1/n)^{n-1} \geq 1/(en)$. If $x_i=1$, the generated solution $\bm{x}^i$ now has the smallest weight $X_t-w_i$, and must be included into the archive $M$ (otherwise, the minimum weight of the current solutions has been smaller than $X_t$, leading to a contradiction), implying $X_{t+1}=X_t-w_i$. Thus, we have
\begin{align}
&E(X_{t+1}\mid X_t) \leq X_t-\sum\limits_{i:x_i=1} \frac{w_i}{en(m+1)}.
\end{align}
Because $X_t=w(\bm{x})=\sum\nolimits_{i:x_i=1}w_i$, we have $E(X_t-X_{t+1}\mid X_t) \geq X_t/(en(m+1))$. Furthermore, $X_0 \leq nw_{\max}$, and the weight of a non-empty solution is at least $w_{\min}=\min\{w_i\mid 1\leq i\leq n\}$, i.e., the minimum weight of a single set $S_i$. According to the multiplicative drift analysis theorem~\cite{doerr2012multiplicative}, we can derive that the expected running time to make $X_t=0$ (i.e., find the solution $\bm{0}$) is at most $en(m+1)(1+\ln (nw_{\max}/w_{\min}))\in O(mn(\log n + \log (w_{\max}/w_{\min})))$.

Let $H_i=\sum^i_{j=1} 1/j$ denote the $i$-th harmonic number, and let $R_k=H_{m}-H_{m-k}$. For phase (2), we consider a quantity $K_{\max}$, which is defined as the maximum value of $k$ such that there exists a solution $\bm{x}$ in the archive $M$ with $c(\bm{x})=k$ and $w(\bm{x}) \leq R_k \cdot \mathrm{OPT}$. Then, we only need to analyze the expected running time until $K_{\max}=m$, implying having found a solution with $c(\bm{x})=m$ (i.e., covering all the elements) and $w(\bm{x}) \leq R_m \cdot \mathrm{OPT}=H_m \cdot \mathrm{OPT}<(\ln m +1)\cdot \mathrm{OPT}$.

After finding the solution $\bm{0}$ with $c(\bm{0})=0$ and $w(\bm{0})=0$, we have $K_{\max} \geq 0$. Assume that currently $K_{\max}=k < m$, and let $\bm{x}$ denote the corresponding solution with $c(\bm{x})=k$ and $w(\bm{x}) \leq R_k \cdot \mathrm{OPT}$. Because $\bm{x}$ can only be replaced by a solution which covers the same number of elements and has a smaller weight, $K_{\max}$ cannot decrease. Next, we are to show that $K_{\max}$ can increase by flipping a specific 0 bit of $\bm{x}$. Let $\bm{x}^i$ denote the solution generated by flipping the $i$-th bit of $\bm{x}$. Let $\delta=\min\{\frac{w_{i}}{c(\bm{x^{i}})-c(\bm{x})}\mid c(\bm{x}^i)-c(\bm{x})>0\}$, which must satisfy $\delta \leq \mathrm{OPT}/(m-k)$. Otherwise, for any $S_i \in \bm{x}^*\setminus \bm{x}$, $w_i> (c(\bm{x}^i)-c(\bm{x})) \cdot \mathrm{OPT}/(m-k)$, where $\bm{x}^*$ denotes an optimal solution, i.e., $c(\bm{x}^*)=m$ and $w(\bm{x}^*)=\mathrm{OPT}$; then, $\sum_{S_i \in \bm{x}^*\setminus \bm{x}} w_i > \big(\sum_{S_i \in \bm{x}^*\setminus \bm{x}}(c(\bm{x}^i)-c(\bm{x}))\big) \cdot \mathrm{OPT}/(m-k)$. Because $c(\cdot)$ denotes the number of elements covered by a solution, we have $\sum_{S_i \in \bm{x}^*\setminus \bm{x}}(c(\bm{x}^i)-c(\bm{x})) \geq c(\bm{x}^*)-c(\bm{x})$, and thus $\sum_{S_i \in \bm{x}^*\setminus \bm{x}} w_i > (c(\bm{x}^*)-c(\bm{x})) \cdot \mathrm{OPT}/(m-k) =\mathrm{OPT}$, where the equality holds by $c(\bm{x}^*)=m$ and $c(\bm{x})=k$. This contradicts with $\sum_{S_i \in \bm{x}^*\setminus \bm{x}} w_i \leq w(\bm{x}^*)= \mathrm{OPT}$. Thus, it holds that $\delta \leq \mathrm{OPT}/(m-k)$. By selecting $\bm{x}$ in line~7 of Algorithm~\ref{algo:map-elites} and flipping only the 0 bit corresponding to $\delta$ in line~8, MAP-Elites can generate a new offspring solution $\bm{x}'$ with $c(\bm{x}')=k'>k$ and
\begin{align}
w(\bm{x}') &\leq w(\bm{x})+(k'-k) \cdot \mathrm{OPT}/(m-k) \\
&\leq R_k \cdot \mathrm{OPT}+(k'-k) \cdot \mathrm{OPT}/(m-k)\\
&=\left(\frac{1}{m}+\cdots+\frac{1}{m-k+1}+\frac{k'-k}{m-k}\right)\cdot \mathrm{OPT}\\
&\leq (1/m+\cdots+1/(m-k'+1))\cdot \mathrm{OPT}\\
&=(H_m-H_{m-k'})\cdot \mathrm{OPT}= R_{k'}\cdot \mathrm{OPT}.
\end{align}
Once generated, $\bm{x}'$ will be included into the archive $M$. Otherwise, there has existed a solution in $M$ which covers $k'$ elements and has a weight no larger than $w(\bm{x}')$ and thus $R_{k'}\cdot \mathrm{OPT}$; this implies that $K_{\max}$ has already been larger than $k$, contradicting with the assumption $K_{\max}=k$. After including $\bm{x}'$, $K_{\max}$ increases from $k$ to $k'$. Because the probability of selecting $\bm{x}$ and flipping only a specific 0 bit in each iteration of MAP-Elites is at least $(1/|M|) \cdot (1/n)(1-1/n)^{n-1} \geq 1/(en(m+1))$, the expected running time for increasing $K_{\max}$ is at most $en(m+1)$. Since it is sufficient to increase $K_{\max}$ at most $m$ times for making $K_{\max}=m$, the expected running time of phase~(2) is $O(m^2n)$.

By combining the above two phases, the expected running time of the whole process is $O(mn(\log n + \log (w_{\max}/w_{\min}))+m^2n)=O(mn(m+\log n + \log (w_{\max}/w_{\min})))$. Thus, the theorem holds.
\end{proof}

\subsection{Analysis of $(\mu+1)$-EA}

Next, we will give a concrete instance of set cover, where $(\mu+1)$-EA fails to achieve a good approximation ratio in polynomial time. To apply $(\mu+1)$-EA in Algorithm~\ref{algo:ea} to solve the set cover problem, the population size $\mu$ is set to the number of cells (i.e., $m+1$) contained by the archive $M$ of MAP-Elites for fair comparison. The considered set cover instance is obtained from the complete bipartite graph in Figure~\ref{fig_bg}(b), and presented in Example~\ref{example-sc}. 

\begin{example}\label{example-sc}
The ground set $U$ contains all the $n-1$ edges of the complete bipartite graph in Figure~\ref{fig_bg}(b). Each set $S_i$ in $V$ consists of the edges adjacent to the vertex $v_i$, i.e., $S_1=\{(v_1,v_2),\ldots,(v_1,v_n)\}$; $\forall 2 \leq i \leq n$, $S_i=\{(v_i,v_1)\}$, where $(v_i,v_j)$ denotes the edge between the vertexes $v_i$ and $v_j$. Their weights are: $w_1=2^n$; $\forall 2 \leq i \leq n$, $w_i=1$.
\end{example}

In this example, the size $m$ of the ground set $U$ is $n-1$, and the maximum and minimum weights of a single set $S_i$ are $w_{\max}=2^n$, and $w_{\min}=1$, respectively. The unique optimal solution is $\bm{x}^*=01^{n-1}$ (i.e., $\{S_2,\ldots,S_n\}$), which covers all the edges, and has the minimum weight $w(\bm{x}^*)=n-1$. We can also observe that $\bm{x}_{\mathrm{local}}=10^{n-1}$ (i.e., containing only $S_1$) is local optimal, with the weight $w(\bm{x}_{\mathrm{local}})=2^n$. Furthermore, $\bm{x}_{\mathrm{local}}$ has the smallest weight except $\bm{x}^*$, implying that it needs to flip all the $n$ bits to improve $\bm{x}_{\mathrm{local}}$. The approximation ratio of $\bm{x}_{\mathrm{local}}$ is $w(\bm{x}_{\mathrm{local}})/w(\bm{x}^*)=2^n/(n-1)=2^{m+1}/m$. Note that if a solution does not cover all the edges, i.e., does not satisfy the constraint, its objective value will be penalized, as shown by the reformulated equivalent objective function in Eq.~(\refeq{eq-sc}).

We prove in Theorem~\ref{theo_ea_sc} that on the set cover instance in Example~\ref{example-sc}, $(\mu+1)$-EA requires at least exponential (w.r.t. $m$, $n$, and $\log(w_{\max}/w_{\min})$) expected running time to achieve an approximation ratio smaller than $2^{m+1}/m$. The proof idea is similar to that of Theorem~\ref{theo_ea_mc}, i.e., to show that there is some probability that the population of $(\mu+1)$-EA will get trapped in the local optimal solution $\bm{x}_{\mathrm{local}}$, escaping from which requires all the $n$ bits to be flipped simultaneously.

\begin{theorem}\label{theo_ea_sc}
There exists an instance of set cover (more precisely, Example~\ref{example-sc}), where the expected running time of $(\mu+1)$-EA with $\mu=m+1$ for achieving an approximation ratio smaller than $2^{m+1}/m$ is at least exponential w.r.t. $m$, $n$, and $\log(w_{\max}/w_{\min})$.
\end{theorem}
\begin{proof}
As the population size $\mu=m+1=n$, $(\mu+1)$-EA generates $n$ initial solutions randomly by uniform sampling from $\{0,1\}^n$ in line~1 of Algorithm~\ref{algo:ea}. Let $A$ denote the event that one initial solution is the local optimal solution $\bm{x}_{\mathrm{local}}$, and the other $n-1$ initial solutions are all worse than $\bm{x}_{\mathrm{local}}$. Because $\bm{x}_{\mathrm{local}}$ has the best objective value except the optimal solution $\bm{x}^*$, the event $A$ happens with probability 
$$
P(A)= \frac{n}{2^n}\cdot \left(1-\frac{2}{2^n}\right)^{n-1} = \frac{n}{2^n}(1-o(1)),
$$
where the factor $n$ is because $\bm{x}_{\mathrm{local}}$ can be any of the $n$ initial solutions. Given the event $A$, we then consider that in the first $n-1$ iterations of $(\mu+1)$-EA, $\bm{x}_{\mathrm{local}}$ is always selected from the current population in line~3 of Algorithm~\ref{algo:ea}, and no bits are flipped by bit-wise mutation in line~4. This event (denoted as $B$) will make the population of $(\mu+1)$-EA consist of $n$ copies of $\bm{x}_{\mathrm{local}}$, because a copy of $\bm{x}_{\mathrm{local}}$ will replace a worse solution in the population after each iteration. The event $B$ happens with probability
\begin{align}
P(B)&=\prod^{n-1}_{i=1}\left(\frac{i}{n}\left(1-\frac{1}{n}\right)^{n}\right)\geq \frac{\sqrt{2\pi (n-1)}}{e(2e^2)^{n-1}},
\end{align}
where $i/n$ is the probability of selecting $\bm{x}_{\mathrm{local}}$ from the population in the $i$-th iteration (note that the population contains $i$ copies of $\bm{x}_{\mathrm{local}}$ in the $i$-th iteration), $(1-1/n)^n$ is the probability of keeping all the $n$ bits of $\bm{x}_{\mathrm{local}}$ unchanged in mutation, and the inequality holds by $(n-1)!\geq \sqrt{2\pi (n-1)}((n-1)/e)^{n-1}$ due to Stirling's formula, $((n-1)/n)^{n-1}=(1-1/n)^{n-1} \geq 1/e$ and $(1-1/n)^n\geq 1/(2e)$. 

After the events $A$ and $B$ happen, the population consists of $n$ copies of $\bm{x}_{\mathrm{local}}$. Because $\bm{x}_{\mathrm{local}}=10^{n-1}$ has the best objective value except the optimal solution $\bm{x}^*=1^{n-1}0$, it needs to flip all the $n$ bits to make an improvement, the probability of which is $1/n^n$. Thus, given the events $A$ and $B$, the expected running time to generate a solution better than $\bm{x}_{\mathrm{local}}$ is $n^n$. Because the approximation ratio of $\bm{x}_{\mathrm{local}}$ is $2^n/(n-1)$, $P(A) \geq n(1-o(1))/2^n$, and $P(B)\geq \sqrt{2\pi (n-1)}/(e(2e^2)^{n-1})$, the expected running time of $(\mu+1)$-EA for achieving an approximation ratio smaller than $2^n/(n-1)$ is at least 
$$
\frac{n}{2^n}(1-o(1))\cdot \frac{\sqrt{2\pi (n-1)}}{e(2e^2)^{n-1}} \cdot n^n,
$$
which is exponential w.r.t. $n$. In Example~\ref{example-sc}, we have $m=n-1$ and $\log(w_{\max}/w_{\min})=\log(2^n/1)=n$, implying that the theorem holds.
\end{proof}

By comparing Theorems~\ref{theo-sc} and~\ref{theo_ea_sc}, MAP-Elites is clearly superior to $(\mu+1)$-EA for solving the set cover problem. MAP-Elites can always achieve the $(\ln m +1)$-approximation ratio in expected polynomial time, while there exists an instance where $(\mu+1)$-EA requires at least expected exponential time to achieve an approximation ratio smaller than $2^{m+1}/m$. As we have found from the analysis on the problem of monotone submodular maximization with a size constraint in the last section, the simultaneous search for high-performing solutions at diverse regions in the behavior space makes MAP-Elites better. More specifically, MAP-Elites obtains a good solution covering all elements by following a path across different cells of the behavior space (i.e., gradually covering more elements), while $(\mu+1)$-EA may lose diversity in the population and get trapped in local optima in the space of solutions covering all elements.

\section{Conclusion and Discussion}



This paper continues the line of theoretical studies on QD algorithms which have only just begun. Specifically, it contributes to providing theoretical justification for the benefit of QD algorithms, i.e., bringing better optimization. On the two NP-hard problems, i.e., monotone approximately submodular maximization with a size constraint, and set cover, we prove that MAP-Elites can achieve the (asymptotically) optimal polynomial-time approximation ratio, while $(\mu+1)$-EA requires exponential expected time on some instances. The simultaneous search of QD algorithms for high-performing solutions with diverse behaviors brings better exploration of the search space, helping avoid local optima and find better solutions. In fact, we can find more theoretical evidences by comparing the results from previous separate works on QD and EA's runtime analysis, though not completely fair. For example, Bossek and Sudholt~\shortcite{bossek2023runtime} proved that for maximizing any monotone pseudo-Boolean function over $\{0,1\}^n$, MAP-Elites that generates one random initial solution and uses the number of 1-bits of a solution as the behavior descriptor can find an optimal solution in $O(n^2 \log n)$ expected time, while Lengler and Zou~\shortcite{lengler2021exponential} proved that $(\mu+1)$-EA with $\mu_0 \leq \mu \leq n$ (where $\mu_0$ is some constant) needs superpolynomial time to optimize some monotone functions. In the future, we will theoretically verify the other benefit of QD algorithms that they can better create a map of high-performing solutions in the behavior space than separately searching for a high-performing solution in each cell of the behavior space, by analyzing the QD-Score metric, i.e., the sum of objective values across all obtained solutions in the archive.

There are many interesting theoretical works to be done, e.g., comparing different QD frameworks (MAP-Elites vs. NSLC), and studying the influence of different components (e.g., parent selection strategies and variation operators) of QD algorithms. The findings may inspire the design of better QD algorithms in practice. Another very interesting work is to study the relationship between MAP-Elites and the corresponding multi-objective EA, GSEMO, as also pointed out in~\cite{bossek2023runtime}. By treating the behavior descriptors as the extra objective functions to be optimized, GSEMO behaves somewhat similarly to MAP-Elites. The differences are: 1)~MAP-Elites only compares solutions within the same cell (i.e., having the same behavior descriptor values) while GSEMO compares different behavior descriptor values based on the domination relationship; 2)~MAP-Elites can control the granularity of the behavior space by setting the number of behavior descriptor values in a cell. The currently known results of MAP-Elites (including~\cite{bossek2023runtime} and ours) match that of GSEMO in~\cite{neumann2006minimum,friedrich2010approximating,friedrich2015maximizing,qian2019maximizing}. Thus, it would be interesting to show their performance gap on some problems, especially that MAP-Elites may overcome some difficulty of GSEMO due to its good diversity mechanism.


\section*{Acknowledgments}

This work was supported by the National Science and Technology Major Project (2022ZD0116600) and National Science Foundation of China (62276124). Chao Qian is the corresponding author. The conference version of this paper has appeared at IJCAI'24.

\bibliographystyle{named}
\bibliography{QD-Theory-arXiv}

\end{document}